%% file: main.tex
\documentclass[nohyperref]{article}

\usepackage{microtype}
\usepackage{graphicx}
\usepackage{subfigure}
\usepackage{booktabs} 
\usepackage{multirow}
\usepackage[dvipsnames]{xcolor}

\usepackage{hyperref}

\usepackage[accepted]{icml2022}

\usepackage{amsmath}
\usepackage{amssymb}
\usepackage{mathtools}
\usepackage{amsthm}
\usepackage{bbm}

\usepackage[capitalize,noabbrev]{cleveref}

\theoremstyle{plain}
\newtheorem{theorem}{Theorem}[section]

\newtheorem{corollary}[theorem]{Corollary}
\theoremstyle{definition}

\theoremstyle{remark}

\usepackage[textsize=tiny]{todonotes}

\definecolor{mypink}{cmyk}{0, 0.7808, 0.4429, 0.1412}
\newcommand{\fyred}[1]{\textcolor{mypink}{#1}}

\newcommand\norm[1]{\left\lVert#1\right\rVert}
\input{math_commands.tex}

\icmltitlerunning{Scaling Structured Inference with Randomization}

\begin{document}

\twocolumn[
\icmltitle{Scaling Structured Inference with Randomization}

\icmlsetsymbol{equal}{*}

\begin{icmlauthorlist}
\icmlauthor{Yao Fu}{ed}
\icmlauthor{John P. Cunningham}{columbia,zuck}
\icmlauthor{Mirella Lapata}{ed}\\
\end{icmlauthorlist}

\icmlaffiliation{ed}{School of Informatics, University of Edinburgh}
\icmlaffiliation{columbia}{Statistics Department, Columbia University}
\icmlaffiliation{zuck}{Zuckerman Institute, Columbia University}

\icmlcorrespondingauthor{Yao Fu}{yao.fu@ed.ac.uk}
\icmlcorrespondingauthor{John P. Cunningham}{jpc2181@columbia.edu}
\icmlcorrespondingauthor{Mirella Lapata}{mlap@inf.ed.ac.uk}

\icmlkeywords{Machine Learning, ICML}

\vskip 0.3in
]


\printAffiliationsAndNotice{}

\begin{abstract}
  Deep discrete structured models have seen considerable progress
  recently, but traditional inference using dynamic programming (DP)
  typically works with a small number of states (less than hundreds),
  which severely limits model capacity.  At the same time, across
  machine learning, there is a recent trend of using randomized
  truncation techniques to accelerate computations involving large
  sums.  Here, we propose \emph{a family of randomized dynamic
  programming (RDP) algorithms} for scaling structured models to tens
  of thousands of latent states.  Our method is widely applicable to
  classical DP-based inference (partition function, marginal,
  reparameterization, entropy) and different graph structures (chains,
  trees, and more general hypergraphs).  It is also compatible with
  automatic differentiation: it can be integrated with neural networks
  seamlessly and learned with gradient-based optimizers.  Our core
  technique approximates the sum-product by restricting and
  reweighting DP on a small subset of nodes, which reduces computation
  by orders of magnitude.  We further achieve low bias and variance
  via Rao-Blackwellization and importance sampling.  Experiments over
  different graphs demonstrate the accuracy and efficiency of our
  approach.  RDP can also be used to learn a structured variational autoencoder with a
  scaled inference network which outperforms baselines in terms of
  test likelihood and successfully prevents collapse.
\end{abstract}


\section{Introduction}
\label{sec:intro}
\input{011_intro_tab_algorithms.tex}

\input{010_intro.tex}

\section{Background and Preliminaries}
\label{sec:background}

\input{020_background.tex}

\input{031_method_fig_method.tex}
\section{Scaling Inference with Randomization}
\label{sec:method}

\input{030_method.tex}

\input{041_exp_tab_overall.tex}
\input{041_exp_tab_biasvar.tex}
\section{Scaling Structured VAE with RDP}
\label{sec:vae_example}
\input{032_vae_example.tex}

\section{Experiments}
\label{sec:experiments}
\input{040_experiments.tex}

\section{Conclusion}
\label{sec:conclusion}
\input{050_conclusion.tex}

\section*{Acknowledgements}
We thank Hao Tang and Ivan Titov for insightful comments on the practical issues of automatic differentiation. 
ML is supported by European Research Council
(ERC CoG TransModal 681760)  and EPSRC (grant no EP/W002876/1).
JPC is supported by the Simons Foundation, McKnight Foundation, Grossman Center
for the Statistics of Mind, and Gatsby Charitable Trust.

\bibliography{RDP}
\bibliographystyle{icml2022}

\newpage
\appendix
\onecolumn
\input{060_appendix}

\end{document}

%% file: math_commands.tex
\usepackage{amsmath,amsfonts,bm}

\def\eqref#1{equation~\ref{#1}}

\def\1{\bm{1}}

\def\va{{\bm{a}}}

\def\ve{{\bm{e}}}

\def\vq{{\bm{q}}}
\def\vr{{\bm{r}}}

\def\vx{{\bm{x}}}
\def\vy{{\bm{y}}}

\def\mX{{\bm{X}}}

\DeclareMathAlphabet{\mathsfit}{\encodingdefault}{\sfdefault}{m}{sl}
\SetMathAlphabet{\mathsfit}{bold}{\encodingdefault}{\sfdefault}{bx}{n}

\newcommand{\softmax}{\mathrm{softmax}}



%% file: 011_intro_tab_algorithms.tex
\begin{table*}[t!]
  \caption{\label{tab:intro:algorithms} 
  Dynamic programming algorithms for different inference over different graphs. 
  Our randomization technique covers a spectrum of classical DP
  algorithms on different graphs.    Scaled algorithms shown in
  \fyred{red} are randomized in this work.  
} 
  \begin{small}
  \begin{center}
    \begin{tabular}{@{}lcccc@{}}   
    \toprule
      Graph & Model & Partition \& Marginal & Entropy & Sampling \& Reparameterization  \\
      \midrule 
      Chains & HMM, CRF, Semi-Markov  & \fyred{Forward-Backward} & \fyred{Backward Entropy} & \fyred{Gumbel Sampling}~\citep{Fu2020GumbelCRF}  \\ 
      Hypertrees &  PCFG, Dependency CRF  & \fyred{Inside-Outside} & Outside Entropy & Stochastic Softmax~\citep{paulus2020gradient}  \\ 
      General graph & General Exponential Family &   \fyred{Sum-Product} & \fyred{Bethe Entropy} & Stochastic Softmax~\citep{paulus2020gradient}  \\ 
    \bottomrule 
    \end{tabular}
  \end{center}
\end{small}
\end{table*}


%% file: 010_intro.tex
Deep discrete structured models~\citep{martins-etal-2019-latent, rush2020torch}
have enjoyed great progress recently, improving performance and
interpretability in a range of tasks including sequence
tagging~\citep{ma-hovy-2016-end}, parsing~\citep{zhang2020efficient,
yang2021pcfgs}, and text
generation~\citep{wiseman-etal-2018-learning}.  However, their
capacity is limited by issues of
scaling~\citep{sun2019fast,chiu2020scaling,yang2021pcfgs,
chiu2021lowrank}.  Traditional dynamic programming based inference for
exponential families has limited scalability with large combinatorial
spaces.  When integrating exponential families with neural networks
(e.g., a VAE with a CRF inference network, detailed later), the small
latent combinatorial space severely restricts model capacity.
Existing work has already observed improved performance by scaling
certain types of structures~\citep{yang2021pcfgs, li2020optimus,
chiu2020scaling}, and researchers are eager to know if there are
general techniques for scaling classical structured models.

Challenges in scaling structured models primarily come from memory
complexity.  For example, consider linear-chain
CRFs~\citep{sutton2006introduction}, the classical sequence model that
uses the Forward algorithm, a dynamic programming algorithm, for
computing the partition function exactly.
This algorithm requires $O(TN^2)$ computation where $N$ is the number of latent states and $T$ is the length of the sequence. It is precisely the $N^2$ term that is problematic in terms of memory and computation. 
This limitation is more severe under automatic differentiation (AD) frameworks as all intermediate DP computations are stored for gradient construction. 
Generally, DP-based inference algorithms are not optimized for modern computational hardware like GPUs and typically work under small-data regimes, with $N$ in the range [10, 100]~\citep{ma-hovy-2016-end,wiseman-etal-2018-learning}. 
With larger $N$, inference becomes intractable since the computation graph does not easily fit into GPU memory~\citep{sun2019fast}.

Aligning with a recent trend of exploiting randomization techniques for machine learning problems~\citep{oktay2020randomized, pmlr-v139-potapczynski21a, pmlr-v97-beatson19a}, 
this work proposes a randomization framework for scaling structured models, which encompasses a family of randomized dynamic programming algorithms with a wide coverage of different structures and inference (Table~\ref{tab:intro:algorithms}).
Within our randomization framework, 
instead of summing over all possible combinations of latent states, we only sum over paths with the most probable states and sample a subset of less likely paths to correct the bias according to a reasonable proposal.
Since we only calculate the chosen paths, memory consumption can be reduced to a reasonably small budget.
We thus recast the computation
challenge into a tradeoff between memory budget, proposal accuracy,
and estimation error.  In practice, we show RDP scales existing models
by \textit{two orders of magnitude} with memory complexity \textit{as
small as one percent}.

In addition to the significantly increased scale, we highlight the following advantages of RDP:  
(1) applicability to different structures (chains, trees, and hypergraphs) and inference operations (partition function, marginal, reparameterization, and entropy);  
(2) compatibility with automatic differentiation and existing efficient libraries (like Torch-Struct in~\citealp{rush2020torch}); and
(3) statistically principled controllability of bias and variance. 
As a concrete application, we show that RDP can be used for learning a structured VAE with a scaled inference network.  
In experiments, we first demonstrate that RDP algorithms estimate
partition function and entropy for chains and hypertrees with lower mean square
errors than baselines.  Then, we show their joint effectiveness for
learning the scaled VAE.  RDP outperforms baselines in terms of test
likelihood and successfully prevents posterior collapse.
Our implementation is at \url{https://github.com/FranxYao/RDP}.

%% file: 020_background.tex
\textbf{Problem Statement}\quad We start with a widely-used
structured VAE framework.  Let $\vy$ be an observed variable.
Let~$\mX$ be any latent structure (sequences of latent tags, parse
trees, or general latent graphs, see Table~\ref{tab:intro:algorithms})
that generates $\vy$.  Let~$p_\psi$ denote the generative model,
and~$q_\theta$ the inference model.  We optimize:
\begin{align}
  \mathcal{L} = \mathbb{E}_{q_\theta(\mX | \vy)}[\log p_\psi(\mX, \vy) - \log q_\theta(\mX | \vy)] \label{eq:elbo}
\end{align} 
where the inference model $q_\theta$ 
takes the form of a discrete undirected exponential family (e.g.,
linear-chain CRFs in~\citealp{Fu2020GumbelCRF}).  Successful
applications of this framework include sequence
tagging~\citep{mensch2018differentiable}, text
generation~\citep{li-rush-2020-posterior}, constituency
parsing~\citep{kim2019unsupervised}, dependency
parsing~\citep{corro2018differentiable}, latent tree
induction~\citep{paulus2020gradient}, and so on.  Our goal is to learn
this model with a scaled latent space (e.g.,~a CRF encoder with tens
of thousands of latent states).

\textbf{Sum-Product Recap}\quad 
Learning models in Eq.~\ref{eq:elbo} usually requires 
classical sum-product inference (and its variants) of $q_\theta$, which we now review. 
Suppose $q_\theta$ is in standard overparameterization of a discrete exponential family~\citep{wainwright2008graphical}:
\begin{align}
  q_\theta(\mX) \propto \exp \left\{\sum_{s \in V}\phi_s(x_s) + \sum_{(s, t) \in E} \phi_{st}(x_s, x_t)\right\}
\end{align}
where $\mX = [X_1, ..., X_M]$ is a random vector of nodes in a
graphical model with each node taking discrete values $X_i \in \{1, 2,
..., N\}$, $N$ is the number of states, $\phi_{st}$ is the edge
potential, and $\phi_s$ is the node potential.  Here, we use a general
notation of nodes $V$ and edges $E$. We discuss specific structures
later.  Suppose we want to compute its marginals and log partition function.
The solution is the sum-product algorithm that recursively updates the
message at each edge:
\begin{align}
  \mu_{ts}(x_s) \propto \sum_{x_t' = 1}^N \Big\{ \phi_{st}(x_s, x_t')\phi_t(x_t') \prod_{u \in V_t^s} \mu_{ut}(x_t') \Big\} \label{eq:message}
\end{align}
where $\mu_{ts}(x_s)$ denotes the message from node $t$ to node $s$
evaluated at $X_s = x_s$, $V_t^s$ denotes the set of neighbor nodes of
$t$ except $s$.  Upon convergence, it gives us the Bethe approximation
(if the graph is loopy) of the marginal, partition function, and entropy (as
functions of the edge marginals~$\mu_{ts}$).  These quantities are
then used in gradient based updates of the potentials~$\phi$ (detailed
in Sec.~\ref{sec:vae_example}).  When the underlying graph is a tree,
this approach becomes exact and convergence is linear in the
number of edges.

\textbf{Challenges in Scaling}\quad
The challenge is how to scale the sum-product computation for $q_\theta$. 
Specifically, gradient-based learning requires three types of
inference: (a)~partition estimation (for maximizing likelihood);
(b)~re-parameterized sampling (for gradient estimation); and
(c)~entropy estimation (for regularization).
Existing sum-product variants (Table~\ref{tab:intro:algorithms})
provide exact solutions, but only for a small latent space (e.g.,~a
linear-chain CRF with state number smaller than 100).  Since the
complexity of DP-based inference is usually at least quadratic to the
size of the latent states, it would induce memory overflow if we
wanted to scale it to tens of thousands.

\textbf{Restrictions from Automatic Differentiation}\quad
At first sight, one may think the memory requirement for some
algorithms is not so large.  For example, the memory complexity of a
batched forward sum-product (on chains) is $O(BTN^2)$.  Consider batch
size~$B=10$, sequence length $T=100$, number of states $N=1000$, then
the complexity is $O(10^{9})$, which seems acceptable, especially
given multiple implementation tricks (e.g.,~dropping the intermediate
variables).
However, this is not the case \textit{under automatic
  differentiation}, primarily because  AD requires \textit{storing
  all intermediate computations}\footnote{More recently, there are works improving the memory-efficiency of AD like~\citet{10.5555/3433701.3433727} by off-loading intermediate variables to CPUs. Yet their application on sum-product inference requires modification of the AD library, which raises significant engineering challenges.}
  for building the adjacent gradient
graph.  This not only invalidates tricks like dropping intermediate
variables (because they should be stored)
but also multiplies memory complexity when building the adjacent
gradient graph~\citep{eisner-2016-inside}.  This problem is more
severe if we want to compute higher-order gradients, or if the
underlying DP has higher-order complexity (e.g., the Inside algorithm
of \textit{cubic complexity} $O(T^2N^3)$).
Since gradient-based optimization requires AD-compatibility, 
scaling techniques are under similar restrictions, namely storing all
computations to construct the adjacent gradient graph.

\textbf{Previous Efforts}\quad Certainly there exist several excellent scaling
techniques for structured models, though most come with some
intrinsic limitations that we hope to alleviate.
In addition to AD-compatibility restrictions, many existing techniques
either require additional assumptions (e.g., sparsity
in~\citealp{lavergne2010practical, Sokolovska2010EfficientLO,
  correia2020efficient}, pre-clustering in~\citealp{chiu2020scaling},
or low-rank in~\citealp{chiu2021lowrank}), rely on handcrafted
heuristics for bias correction~\citep{jeong2009efficient}, or cannot
be easily adapted to modern GPUs with tensorization and
parallelization~\citep{klein2003parsing}.  As a result, existing
methods apply to a limited range of models: \citet{chiu2020scaling}
only consider chains and \citet{yang2021pcfgs} only consider probabilistic context-free grammars (PCFGs).

One notably promising direction comes from~\citet{sun2019fast}, where
they consider topK computation and drop all low probability states. 
\citet{NEURIPS2018_6211080f} also considered topK inference in a smoothed setting (where they have more theoretical analytics but are less about memory efficiency) and their topK inference also points to efficiency improvements.
While intuitively sensible
(indeed we build on this idea here), their deterministic truncation
induces severe bias (later shown in our experiments). As we will
discuss, this bias can be effectively mitigated by randomization.

\textbf{Randomization of Sums in Machine Learning}\quad Randomization
is a long-standing technique in machine
learning~\citep{10.1561/2200000035, doi:10.1137/090771806} and has
been applied to dynamic programs before the advent of deep
learning~\citep{NIPS2009_e515df0d, blunsom-cohn-2010-inducing}.
Notably, works like~\citet{koller1999general, pmlr-v5-ihler09a} also consider sampling techniques for sum-product inference.
However, since these methods are proposed before deep learning, their differentiability remain untested. 
More recently, randomization has been used in the context of deep
learning, including density estimation~\citep{NEURIPS2020_33d3b157},
Gaussian processes~\citep{pmlr-v139-potapczynski21a}, automatic
differentiation~\citep{oktay2020randomized}, gradient
estimation~\citep{correia2020efficient}, and
optimization~\citep{pmlr-v97-beatson19a}.  The foundation of our
randomized DP also lies in speeding up summation by randomization.  To
see the simplest case, consider the sum of a sorted list $\va$ of
positive numbers: $S = \sum_{i=1}^N a_i$.  This requires $N-1$
additions, which could be expensive when $N$ is large.  Suppose one
would like to reduce the number of summands to $K$,~\citet{liu2019rao}
discusses the following sum-and-sample estimator for gradient
estimation:
\begin{align}
  \hat{S} &= \sum_{i = 1}^{K_1} a_i + \frac{1}{K_2}\sum_{j=1}^{K_2} \frac{a_{\delta_j}}{q_{\delta_j}} \label{eq:tailsum}
\end{align}
where $K_1 + K_2 = K$ and $\delta_j \sim \vq = [q_{K_1 + 1}, ..., q_{N}]$, 
$\vq$ is a proposal distribution upon the tail summands $[a_{K_1
  + 1}, ..., a_N]$. 
This estimator is visualized in Fig.~\ref{fig:method}A.
One can show that it is unbiased:
$\mathbb{E}_\vq[\hat{S}] = S$, irrespective of how we choose the
proposal $\vq$.  The oracle proposal $\vq^\star$is the normalized tail
summands: $q_i^\star = a_i / \sum_{j = K_1 + 1}^N a_j$, under which
the estimate becomes exact: $\hat{S} \equiv S$.  
Note that the bias is corrected 
by dividing the proposal probability.  


%% file: 031_method_fig_method.tex
\begin{figure*}[ht]
  \centering
  \includegraphics[width=\linewidth]{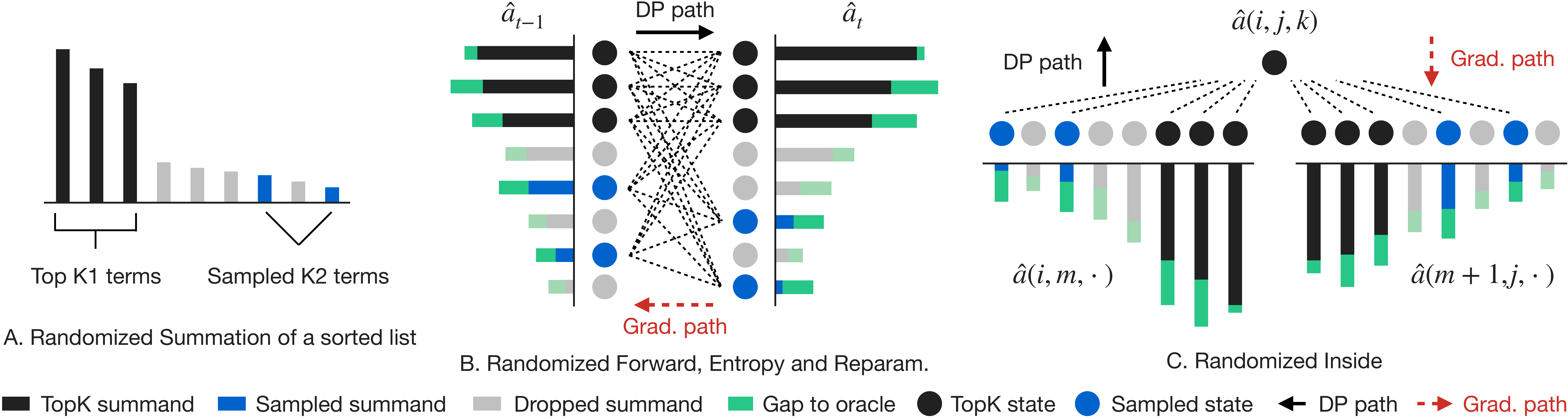}
  \caption{Scaling inference by randomization. 
  (A): randomized summation (Eq.\ref{eq:tailsum}). 
  (B): Randomized Forward (Alg.~\ref{alg:forward}) recursively applies
  randomized summation at each DP step.  Memory reduction is achieved
  by restricting computation on edges linking the sampled nodes.
  Second-order randomized DPs (Entropy and Reparameterization,
  Alg.~\ref{alg:entropy} and~\ref{alg:gumbel_ffbs}) reuse this graph
  and share the same path as the gradients.  (C): Randomized
  Inside (Alg.~\ref{alg:inside}) applies randomized summation twice at
  each DP step.  All algorithms in our randomized DP family are
  compatible to automatic differentiation, the direction of the
  gradient is shown by red dashed arrows.  }
  \label{fig:method}
\end{figure*}

%% file: 030_method.tex
In this section, we introduce our randomized dynamic programming
algorithms.  We first introduce a generic randomized sum-product framework
for approximating the partition function of any graph structure.  
Then we instantiate this framework on
two classical structures, chains and hypertrees, respectively.
When running the randomized sum-product in a left-to-right (or bottom-up) order, 
we get a randomized version of the classicial Forward (or Inside) algorithm for 
estimating the partition function for chains like HMMs and Linear-chain CRFs (or hypertrees like PCFGs and Dependency CRFs in~\citealp{kim2019unsupervised}). 
We call these two algorithms \textit{first-order RDPs} because they run
the computation graph in one pass.  Next, we generalize RDP to two
more inference operations (beyond partition estimation): entropy
estimation (for the purposes of regularized training) and
reparameterization (for gradient-based optimization).  We collectively
use the term \textit{second-order RDPs} to refer to entropy and
reparameterization algorithms because they call first-order RDPs as a
subroutine and run a second pass of computation upon the first-order
graphs. 
We will compose these algorithms into a structured VAE in
Sec.~\ref{sec:vae_example}.

\subsection{Framework: Randomized Sum-Product}
\label{ssec:method:framework}

\input{031_method_framework.tex}


\input{031_method_alg_forward.tex}
\subsection{First-Order Randomized DP}
\label{ssec:method:first_order_rdp}
\input{031_method_first_order.tex}


\input{031_method_alg_entropy.tex}
\input{031_method_alg_ffbs.tex}
\subsection{Second-Order Randomized DP}
\label{ssec:method:second_order_rdp}

\input{031_method_second_order.tex}


%% file: 031_method_framework.tex
Our key technique for scaling the sum-product is a 
\textit{randomized message estimator}. 
Specifically, for each node $X_t$, given a pre-constructed proposal $\vq_t = [q_t(1), ..., q_t(N)]$ (we discuss how to construct a correlated proposal later because it depends on the specific neural parameterization), we retrieve the top $K_1$ index $\{\sigma_{t, i}\}_{i=1}^{K_1}$ from $\vq_t$, and get a $K_2$-sized sample $\{\delta_{t, i}\}_{i=1}^{K_2}$ from the rest of $\vq_t$:
\begin{align}
  &[\sigma_{t, 1}, ..., \sigma_{t, K_1}, ..., \sigma_{t, N}] = \arg \text{sort} \{q_t(i)\}_{i=1}^N \label{eq:topk_index}\\ 
  &\delta_{t, j} \overset{\text{i.i.d.}}{\sim} \text{Categorical}\{q_t(\sigma_{t, K_1 + 1}), ..., q_t(\sigma_{t, N})\} \label{eq:q_proposal}\\ 
  &  \Omega_t^{K_1} = \{\sigma_{t, i}\}_{i=1}^{K_1} \quad \quad \Omega_t^{K_2} = \{\delta_{t, j}\}_{j=1}^{K_2}  \label{eq:topk_and_sample_index}
\end{align}
Then, we substitute the full index $\{1, ..., N\}$ with the the top
$\Omega_t^{K_1}$ and the sampled index $\Omega_t^{K_2}$:
\begin{align}
  & \hat{\mu}_{ts}(x_s) \propto  \sum_{\sigma_t \in \Omega_t^{K_1}} \Big\{ \phi_{st}(x_s, \sigma_t)\phi_t(\sigma_t) \prod_{u \in V_t^s} \hat{\mu}_{ut}(\sigma_t) \Big\} + \notag \\
  & \sum_{\delta_t \in \Omega_t^{K_2}} \Big\{ \frac{1}{K_2 q_t(\delta_t)} \phi_{st}(x_s, \delta_t)\phi_t(\delta_t) \prod_{u \in V_t^s} \hat{\mu}_{ut}(\delta_t) \Big\} \label{eq:sampled_message}
\end{align}
where the oracle proposal is proportional to the actual summands
(Eq.~\ref{eq:message}).  
Here indices are pre-computed outside DP.
This means the computation of
Eq.~\ref{eq:topk_index}--\ref{eq:topk_and_sample_index} can be moved
outside the AD engine (e.g., in Pytorch, under the \texttt{no\_grad}
statement) to further save GPU memory usage.

We now show that estimator Eq.~\ref{eq:sampled_message} can be interpreted as a combination of Rao-Blackwellization (RB) and importance sampling (IS). 
Firstly, RB says that a function $J(X, Y)$ depending on two random variables $X, Y$ has larger variance than its conditional expectation $\hat{J}(X) = \mathbb{E}[J(X, Y) | X]$~\citep{pmlr-v33-ranganath14}: 
\begin{align}
  \mathbb{V}[\hat{J}(X)] &= \mathbb{V}[J(X, Y)] - \mathbb{E}[(J(X, Y) - \hat{J}(X))^2] \\
  &\le \mathbb{V}[J(X, Y)] \label{eq:rb}
\end{align}
This is because $Y$ has been integrated out and the source of randomness is 
reduced to $X$. 
Now let  $\delta_1, \delta_2$ be uniform random indices taking values from $\Omega_t^{K_1}$ and $\Omega_t^{K_2}$, respectively.
Consider a simple unbiased estimate $S$ for the message $\mu_{st}$: 
\begin{align}
  m_{st}(i) &= \phi_{st}(x_s, i)\phi_t(i) \prod_{u \in V_t^s} \mu_{ut}(i) \label{eq:message_def}\\
  S(\delta_1, \delta_2) &= K_1 m_{st}(\delta_1) + (N - K_1) m_{st}(\delta_2)
\end{align}
The conditional expectation $\hat{S}(\delta_2) = \mathbb{E}[S(\delta_1, \delta_2) | \delta_2]$ is:
\begin{align}
  \hat{S}(\delta_2) = \sum_{\sigma_t \in \Omega_t^{K_1}} m_{st}(\sigma_t) + (N - K_1) m_{st}(\delta_2) \label{eq:conditional_message}
\end{align}
Plugging $S(\delta_1, \delta_2), \hat{S}(\delta_2)$ in Eq.~\ref{eq:rb}, we get 
variance reduction:
$\mathbb{V}[\hat{S}(\delta_2)] \le \mathbb{V}[S(\delta_1, \delta_2)]$.
Note that the variance will becomes smaller as $K_1$ becomes larger. 
Now change $\delta_2$ in Eq.~\ref{eq:conditional_message} to from the proposal in Eq.~\ref{eq:q_proposal} then take expectation: 
\begin{align}
  &\mathbb{E}_{\delta_2 \sim \text{Uniform}(\Omega_t^{K_2})}[(N - K_1)m_{st}(\delta_2)]  \\
  &= \mathbb{E}_{\delta_2 \sim \vq_t}[\frac{1}{q_t(\delta_2)(N-K_1)}(N - K_1)m_{st}(\delta_2)] \\
  &= \mathbb{E}_{\delta_2 \sim \vq_t}[\frac{1}{q_t(\delta_2)} m_{st}(\delta_2)] \label{eq:is_message}
\end{align}
This is an instance of importance sampling where the importance weight is 
$\frac{1}{q_t(\delta_2)(N-K_1)}$.
Variance can be reduced if $\vq_t$ is correlated with the summands (Eq.~\ref{eq:message_def}).
Finally, combining Eq.~\ref{eq:is_message} and~\ref{eq:conditional_message} and 
increasing the number of sampled index $\delta_2$ to $K_2$ will recover our full message estimator in Eq.~\ref{eq:sampled_message}.

%% file: 031_method_alg_forward.tex
\begin{algorithm}[t!]
  \small
  \caption{Randomized Forward}
  \label{alg:forward}
\begin{algorithmic}
  \STATE {\bfseries Input:} potentials $\phi(x_{t-1}, x_t, y_t)$,  top $K_1$ index set $\Omega_t^{K_1}$, sampled $K_2$ index set $\Omega_t^{K_2}$
  \STATE \textbf{Initialize} $\alpha_1(i) = \phi(x_0 = \emptyset, x_1 = i, y_t)$
  \STATE \textbf{For} $t=2$ {\bfseries to} $T$, compute recursion:
  \begin{align}
    \hat{\alpha}_t(i) &= \sum_{\sigma \in \Omega_{t-1}^{K_1}} \hat{\alpha}_{t-1}(\sigma)\phi(\sigma, i, y_t) + \notag\\ 
    & \sum_{\delta \in \Omega_{t-1}^{K_2}} \frac{1}{K_2 q_t(\delta)}\hat{\alpha}_{t-1}(\delta)\phi(\delta, i, y_t) \label{eq:forward_recursion}
  \end{align}
  \STATE \textbf{Return} $\hat{Z} = \sum_{\sigma \in \Omega_T^{K_1}} \hat{\alpha}_T(\sigma) + \sum_{\delta \in \Omega_T^{K_2}} \frac{1}{K_2 q_T(\delta)}\hat{\alpha}_T(\delta)$ and $\{\hat{\alpha}_t\}_{t=1}^T$ 
\end{algorithmic}
\end{algorithm}

%% file: 031_method_first_order.tex
Now we instantiate this general randomized message passing principle
for chains (on which sum-product becomes the Forward algorithm) and tree-structured hypergraphs (on which sum-product becomes the Inside algorithm).
When the number of states~$N$ is large,
exact sum-product requires large GPU memory when implemented with AD
libraries like Pytorch.
This is where randomization comes to rescue. 

\subsubsection{Randomized Forward}
Algorithm~\ref{alg:forward} shows our randomized Forward algorithm for approximating the partition function of chain-structured graphs. 
The core recursion in Eq.~\ref{eq:forward_recursion} estimates the
alpha variable~$\hat{\alpha}_t(i)$ as the sum of all possible
sequences up to step $t$ at state $i$.  
It corresponds to the
Eq.~\ref{eq:sampled_message} applied to chains
(Fig.~\ref{fig:method}B). 
Note how the term $K_2q_t(\delta)$ is divided in Eq.~\ref{eq:forward_recursion} to correct the estimation bias. 
Also note that all computation in Alg.~\ref{alg:forward} are differentiable w.r.t. the factors $\phi$.
We can recover the classical Forward
algorithm by changing the chozen index set $\Omega_t^K$ to the full
index $[1, .., N]$.  In Appendix~\ref{sec:app:theory:bias}, we prove
the unbiasedness by induction.



As discussed in the above section,  
we reduce the variance of the randomized Forward by
(1)~Rao-Blackwellization  (increasing $K_1$ to reduce randomness); and
(2)~importance sampling (to construct a proposal correlated to the actual summands). 
The variance comes from the gap between the proposal and the oracle
(only accessible with a full DP), as  shown in the green bars in Fig~\ref{fig:method}B.
Variance is also closely related to how \textit{long-tailed} the underlying distribution is: the longer the tail, the more effective Rao-Blackwellization will be. 
More detailed variance analysis
is presented in Appendix~\ref{sec:app:theory:variance}.  In practice,
we implement the algorithm in log space (for numerical stability)\footnote{Common machine learning practice works with log probabilities, thus log partition functions. Yet one can also implement Alg.\ref{alg:forward} in the original space for unbiasedness.}.
which has two implications: (a). the estimate becomes a lower bound due to
Jensen's inequality;
(b). the variance is \textit{exponentially reduced} by the $\log(\cdot)$ function (in a rather trivial way).  

Essentially, the Sampled Forward restricts the DP computation from the
full graph to a subgraph with chosen nodes ($\Omega_t^{K_1}$ and
$\Omega_t^{K_2}$ for all $t$), quadratically reducing  memory
complexity from $O(TN^2)$ to $O(TK^2)$. 
Since all computations in Alg.~\ref{alg:forward} are differentiable,
one could directly compute gradients of the estimated partition function with
any AD library, thus enabling gradient-based optimization.  
Back-propagation shares the same DP graph as the Randomized Forward,
but reverses its direction (Fig.~\ref{fig:method}B).  

\input{031_method_alg_inside.tex}
\subsubsection{Randomized Inside}
We now turn to our Randomized Inside algorithm for approximating the partition function of tree-structured hypergraphs (Alg.~\ref{alg:inside}).
It recursively estimates the inside variables $\hat{\alpha}(i, j, k)$
which sum over all possible tree branchings (index $m$ in
Eq.~\ref{eq:inside_recursion}) and chosen state combinations (index
$\sigma_1, \sigma_2, \delta_1, \delta_2$, Fig~\ref{fig:method}C).
Index $i, j$ denotes a subtree spanning from location $i$ to $j$ in a given
sequence, and $k$ denotes the state.  
Different from the Forward case,
this algorithm computes the product of
two randomized sums that represent two subtrees, i.e.,~$(i, m)$
and $(m + 1, j)$.  The proposal is constructed for each subtree. i.e., $q_{i, m}$ and $q_{m+1,j}$, and are both divided in Eq.~\ref{eq:inside_recursion} for correcting the estimation bias.  


The Randomized Inside restricts computation to the sampled states of
subtrees, reducing the complexity from $O(T^2N^3)$ to $O(T^2K^3)$.
The output partition function is again differentiable and the gradients reverse
the computation graph.  Similar to the Forward case, unbiasedness can
be proved by induction.
The variance 
can be decreased by  increasing $K_1$ to reduce
randomness and constructing a correlated proposal. 
In Fig.~\ref{fig:method}C, green bars represent gaps to the oracle proposal, which is a major source of estimation error.


%% file: 031_method_alg_inside.tex
\begin{algorithm}[t!]
  \caption{Randomized Inside}
  \small
  \label{alg:inside}
  \begin{algorithmic}
    \STATE {\bfseries Input:} potentials $\phi(i, j, k)$, top $K_1$ index set $\Omega_{i, j}^{K_1}$, sampled $K_2$ index set $\Omega_{i, j}^{K_2}$
    \STATE \textbf{Initialize} $\alpha(i, i, k) = \phi(i, i, k)$.
    \STATE \textbf{For} $l=1$ {\bfseries to} $T - 1$, let $j = i + l$, compute recursion:
    \begin{align}
      \hat{\alpha}&(i, j, k) = \phi(i, j, k) \sum_{m = i}^{j-1}\big\{ \notag\\
       &\sum_{\sigma_1 \in \Omega_{i, m}^{K_1}, \sigma_2 \in \Omega_{m + 1, j}^{K_1} } \hat{\alpha}(i, m, \sigma_1) \cdot \hat{\alpha}(m + 1, j, \sigma_2) + \notag\\
       &\sum_{\delta_1 \in \Omega_{i, m}^{K_2}, \delta_2 \in \Omega_{m + 1, j}^{K_2} } \frac{1}{K_2 q_{i,m}(\delta_1)} \hat{\alpha}(i, m, \delta_1) \cdot \notag\\
       & \quad \quad \quad \quad \quad \frac{1}{K_2 q_{m+1, j}(\delta_2)} \hat{\alpha}(m + 1, j, \delta_2) \big\} \label{eq:inside_recursion}
    \end{align}
    \STATE \textbf{Return} $\hat{Z} = \sum_{\sigma \in \Omega_{1, T}^{K_1}} \hat{\alpha}(1, T, \sigma) + $ \\ $\quad\quad\quad\quad\quad\;\sum_{\delta \in \Omega_{1, T}^{K_2}} \frac{1}{K_2 q_T(\delta)}\hat{\alpha}(1, T, \delta)$  \\
    $\quad\quad\quad \text{and} \;\;\{\hat{\alpha}_{i, j}\}, i, j \in\{1, ..., T\}$
  \end{algorithmic}
\end{algorithm}

%% file: 031_method_alg_entropy.tex
\begin{algorithm}[tb]
  \caption{Randomized Entropy DP}
  \label{alg:entropy}
  \small
  \begin{algorithmic}
    \STATE {\bfseries Input:} potentials $\phi(x_{t-1}, x_t, y_t)$,  top $K_1$ index set $\Omega_t^{K_1}$, sampled $K_2$ index set $\Omega_t^{K_2}$
    \STATE \textbf{Initialize} $H_1(i) = 0$; call Randomized Forward to get $\hat{Z}, \hat{\alpha}$
    \STATE \textbf{For} $t=1$ {\bfseries to} $T-1$, compute recursion:
    \begin{align}
      &\hat{p}_t(i, j) = \phi(i, j, x_t)\hat{\alpha}_{t}(i) / \hat{\alpha}_{t + 1}(j) \\ 
      &\hat{H}_{t + 1}(j) = \sum_{\sigma \in \Omega_{t}^{K_1}} \hat{p}_t(\sigma, j)[\hat{H}_{t}(\sigma) - \log \hat{p}_t(\sigma, j)] + \notag\\ 
      &\;\;  \sum_{\delta \in \Omega_{t}^{K_2}} \frac{1}{K_2 q_{t}(\delta)}\hat{p}_t(\delta, j)[\hat{H}_{t}(\delta) - \log \hat{p}_t(\delta, j)] \label{eq:ent_recursion}
    \end{align}
    \STATE \textbf{Return} $\hat{H}$ where: 
    \begin{align}
      &\hat{p}_T(i) = \hat{\alpha}_T(i) / \hat{Z} \\ 
      &\hat{H} = \sum_{\sigma \in \Omega_{T}^{K_1}} \hat{p}_T(\sigma)[\hat{H}_{T}(\sigma) - \log \hat{p}_T(\sigma)] + \notag\\ 
      & \quad \sum_{\delta \in \Omega_{T}^{K_2}} \frac{1}{K_2 q_{T}(\delta)}\hat{p}_T(\delta)[\hat{H}_{T}(\delta) - \log \hat{p}_T(\delta)]
    \end{align}
  \end{algorithmic}
\end{algorithm}

%% file: 031_method_alg_ffbs.tex
\begin{algorithm}[tb]
  \caption{Randomized Gumbel Backward Sampling}
  \label{alg:gumbel_ffbs}
  \small
  \begin{algorithmic}
    \STATE {\bfseries Input:} potentials $\phi(x_{t-1}, x_t, y_t)$,  top $K_1$ index set $\Omega_t^{K_1}$, sampled $K_2$ index set $\Omega_t^{K_2}$, gumble noise $g_t(i) \sim \text{Gumbel}(0, 1)$
    \STATE \textbf{Initialize:} call Randomized Forward to get $\hat{Z}, \hat{\alpha}$, then: 
    \begin{align}
      &\hat{p}_T(i) = \hat{\alpha}_{T}(i)/ \hat{Z},\;\; i \in \Omega_T^{K_1} \cup \Omega_T^{K_2}\\ 
      &\tilde{\vx}_T = \text{softmax}_i (\log \hat{p}_T(i) + g_T(i))  \\ 
      &\hat{x}_T = \text{argmax}_i \tilde{\vx}_T(i)
    \end{align}
    \STATE \textbf{For} $t=T-1$ {\bfseries to} $1$, compute recursion:
    \begin{align}
      &\hat{p}_{t}(i, j) = \phi(i, j, y_t)\hat{\alpha}_{t}(i) / \hat{\alpha}_{t + 1}(j) \notag\\
      &\quad\quad\quad\quad i \in \Omega_t^{K_1} \cup \Omega_t^{K_2}, j \in \Omega_{t+1}^{K_1} \cup \Omega_{t+1}^{K_2}\\ 
      &\tilde{\vx}_t = \text{softmax}_i (\log \hat{p}_{t}(i, \hat{x}_{t + 1}) + g_t(i))  \\ 
      &\hat{x}_t = \text{argmax}_i \tilde{\vx}_t(i)
    \end{align}
    \STATE \textbf{Return} relaxed sample $\{\tilde{\vx}_t\}_{t=1}^T$, hard sample $\{\hat{x}_t\}_{t=1}^T$
  \end{algorithmic}
\end{algorithm}

%% file: 031_method_second_order.tex
We now generalize RDP from partition function estimation to two more inference operations: entropy and reparameterized sampling. 
When composing graphical models with neural networks, entropy is usually required for regularization, and reparameterized sampling is required for Monte-Carlo gradient estimation~\citep{mohamed2020monte}.
Again, we call our methods \textit{second-order} because they
reuse the computational graph and intermediate outputs of first-order RDPs.
We focus on chain structures (HMMs and linear-chain CRFs) for simplicity. 
 
 
\subsubsection{Randomized Entropy DP}
Algorithm~\ref{alg:entropy} shows the randomized entropy DP\footnote{See \citet{mann-mccallum-2007-efficient, li-eisner-2009-first} for more details on deriving entropy DP.}. 
It reuses the chosen index $\Omega_t^K$ (thus the computation graph) and the intermediate variables $\hat{\alpha}_t$ of the randomized Forward, and recursively estimates the conditional entropy $\hat{H}_t(j)$ which represents the entropy of the chain ending at step $t$, state $j$. 
Unbiasedness can be similarly proved by induction. 
Note that here the estimate is biased because of the $\log(\cdot)$ in Eq.~\ref{eq:ent_recursion} (yet in the experiments we show its bias is significantly less than the baselines).
Also note that the proposal probability $q_t(\delta)$ should be divided for bias correction (Eq.~\ref{eq:ent_recursion}). 
Again, all computation is differentiable, which means that the  
output entropy can be directly differentiated by AD engines. 


\subsubsection{Randomized Gumbel Backward Sampling}
When training VAEs with a structured inference network, one usually requires differentiable samples from the inference network.
Randomized Gumbel backward sampling (Alg.~\ref{alg:gumbel_ffbs}) provides differentiable relaxed samples from HMMs and Linear-chain CRFs. 
Our algorithm is based on the recently proposed Gumbel Forward-Filtering
Backward-Sampling (FFBS) algorithm for reparameterized gradient
estimation of CRFs (see more details in ~\citealp{Fu2020GumbelCRF}), and scales it to CRFs with
tens of thousands of states.  It reuses DP paths of the randomized
Forward and recursively computes hard sample $\hat{x}_t$ and soft
sample $\tilde{\vx}_t$ (a relaxed one-hot vector) based on the chosen
index $\Omega_t^K$.  When differentiating these soft samples for
training structured VAEs, they induce biased but low-variance
reparameterized gradients.



%% file: 041_exp_tab_overall.tex
\begin{table*}[t]
  \caption{
  Mean square error comparison between RDP algorithms v.s. TopK approximation. D = Dense, I = intermediate, L = Long-tailed distributions. 
  $\log Z$ denotes log partition function.
  Our method outperforms the baseline on all unit cases with significantly less memory. 
  }
  \label{tab:exp:overall}
  \small
  \centering
  \begin{tabular}{@{}lccccccccc@{}}
      \toprule
       & \multicolumn{3}{c}{Linear-chain $\log Z$}  & \multicolumn{3}{c}{Hypertree $\log Z$}  & \multicolumn{3}{c}{Linear-chain Entropy} \\
      \cmidrule(lr){2-4}\cmidrule(lr){5-7} \cmidrule(lr){8-10}
      $N=2000$ & D & I & L & D & I & L & D & I & L  \\ \midrule
      \textsc{TopK 20\%$N$}  & 3.874 & 1.015 & 0.162 & 36.127 & 27.435 & 21.78 & 443.7 & 84.35 & 8.011 \\ 
      \textsc{TopK 50\%$N$}  & 0.990 & 0.251 & 0.031 & 2.842 & 2.404 & 2.047 & 131.8 & 22.100 & 1.816 \\ 
      \textsc{RDP \textbf{1\%}$N$} (ours)  & 0.146 & 0.066 & 0.076 & 26.331 & 37.669 & 48.863 & 5.925 & 1.989 & 0.691 \\ 
      \textsc{RDP 10\%$N$} (ours)  & 0.067 & 0.033 & 0.055 & 1.193 & 1.530 & 1.384 & 2.116 & 1.298 & 0.316 \\ 
      \textsc{RDP 20\%$N$} (ours)  & \bf 0.046 & \bf 0.020 & \bf 0.026 & \bf 0.445 & \bf 0.544 & \bf 0.599 & \bf 1.326 & \bf 0.730 & \bf 0.207 \\ \midrule
      $N=10000$ & D & I & L & D & I & L & D & I & L  \\ \midrule
      \textsc{TopK 20\%$N$}       & 6.395 & 6.995 & 6.381 & 78.632 & 63.762 & 43.556 & 227.36 & 171.97 & 141.91    \\ 
      \textsc{TopK 50\%$N$}       & 2.134 & 2.013 & 1.647 & 35.929 & 26.677 & 17.099 & 85.063 & 59.877 & 46.853   \\ 
      \textsc{RDP \textbf{1\%}$N$} (ours)  & 0.078 & 0.616 & 0.734 & 3.376  & 5.012  & 7.256  & 6.450 & 6.379 & 4.150   \\ 
      \textsc{RDP 10\%$N$} (ours) & 0.024 & 0.031 & 0.024 & 0.299  & 0.447  & 0.576  & 0.513 & 1.539 & 0.275   \\ 
      \textsc{RDP 20\%$N$} (ours) & \bf 0.004 & \bf 0.003 & \bf 0.003 & \bf 0.148  & \bf 0.246  & \bf 0.294  & \bf 0.144 & \bf 0.080 & \bf 0.068    \\ 
      \bottomrule
  \end{tabular}
\end{table*}

%% file: 041_exp_tab_biasvar.tex
\begin{table}[t]
  \caption{
  Bias-Variance decomposition of Randomized Forward.
  }
  \label{tab:exp:biasvar}
  \small
  \centering
  \begin{tabular}{@{}lcccc@{}}
      \toprule
       & \multicolumn{2}{c}{Dense}   & \multicolumn{2}{c}{Long-tail} \\
       & Bias & Var. & Bias & Var.  \\ \midrule
      \textsc{TopK 20\%$N$} & -1.968 & 0  & -0.403 & 0 \\ 
      \textsc{TopK 50\%$N$} & -0.995 & 0  & -0.177 & 0  \\ 
      \textsc{RDP 1\%$N$} (ours)  &  -0.066 & 0.141  & -0.050 & 0.074 \\ 
      \textsc{RDP 10\%$N$} (ours)  & -0.030 & 0.066 & -0.027 & 0.054  \\ 
      \textsc{RDP 20\%$N$} (ours)  & -0.013 & 0.046  & -0.003 & 0.026 \\ 
      \bottomrule
  \end{tabular}
\end{table}

%% file: 032_vae_example.tex
We study a concrete structured VAE example that uses our RDP
algorithms for scaling.  We focus on the language domain, however our
method generalizes to other types of sequences and structures.  We will use
the randomized Gumbel backward sampling algorithm for gradient estimation and the
randomized Entropy DP for regularization.  

\textbf{Generative Model}\quad
Let $\vx = [x_1, ..., x_T]$ be a sequence of discrete latent states, and $\vy= [y_1, ..., y_T]$ a sequence of observed words (a word is an observed categorical variable). 
We consider an autoregressive generative model parameterized by an LSTM decoder:
\begin{align}
  &p_\psi(\vx, \vy) = \prod_t p_\psi(x_t | x_{<t}, y_{<t}) \cdot p_\psi(y_t | x_t, x_{<t}, y_{<t}) \notag\\ 
  &p_\psi(x_t | x_{<t}, y_{<t}) = \softmax(\text{MLP}(\text{LSTM}(x_{<t}, y_{<t}))) \notag \\
  &p_\psi(y_t | x_t, x_{<t}, y_{<t}) = \softmax(\text{MLP}(\text{LSTM}(x_{1:t}, y_{<t}))) \notag
\end{align}
\textbf{Inference Model}\quad Since the posterior is intractable, we
use a structured inference model parameterized by a CRF with a neural
encoder:
\begin{align}
  q_\theta(\vx | \vy) = \prod_t \Phi(x_{t-1}, x_t)\phi(x_t, y_t) / Z
\end{align}
where $\Phi$ is the $N \times N$ transition potential matrix and $\phi$ is the emission potential.
This formulation has been previously used in sequence tagging and
sentence generation~\citep{ammar2014conditional,
  li-rush-2020-posterior} for a small~$N$. With large~$N$, 
 parametrization and inference require careful treatment. 
Specifically, we firstly use a neural encoder to map $\vy$ to a set of contextualized representations: 
\begin{align}
  [\vr_1, ..., \vr_T] = \text{Enc}(y_1, ..., y_t)
\end{align}
As direct parametrization of an $N \times N$ transition matrix would be memory expensive when $N$ is large (e.g., when $N=10^5$), we decompose it by associating each state $i$ with a state embedding~$\ve_i$.
Then, the potentials are:
\begin{align}
  \Phi(x_{t-1}, x_t) = \ve_{x_{t-1}}^\intercal \ve_{x_{t}} \quad\quad \phi(x_t, y_t) = \ve_{x_t}^\intercal \vr_t \label{eq:transition_decomp}
\end{align}
Training uses the following format of ELBO:
\begin{align}
  \mathcal{L} = \mathbb{E}_{q_\theta(\vx | \vy)}[\log p_\psi(\vx, \vy)] + H(q_\theta(\vx | \vy))
\end{align}
where the entropy (and its gradients) is estimated by Alg.~\ref{alg:entropy}. 
Stochastic gradients of the first term is obtained by differentiating samples from randomized gumbel sampling:
\begin{align}
  \nabla_\theta \mathbb{E}_{q_\theta}[\log p_\psi(\vx, \vy)] \approx \nabla_\theta \log p_\psi(\tilde{\vx}(\theta), \vy), \tilde{\vx}(\theta) \sim \text{Alg.~\ref{alg:gumbel_ffbs}} \label{eq:reparam_grad}
\end{align} 
note that the reparameterized sample $\tilde{\vx}(\theta)$ is a function of $\theta$.
Equation~\ref{eq:reparam_grad} gives low-variance gradients of the inference network, which is ready for gradient-based optimization.

\input{041_exp_tab_proposal.tex}
\input{041_exp_fig_k1k2.tex}

\textbf{Proposal Construction}\quad 
we consider the following: 
\begin{align}
  q_t(i) = \frac{\phi(i, y_t)}{2 \sum_{j=1}^N \phi(j, y_t)} + \frac{\norm{\ve_i}_1}{2\sum_{j=1}^N \norm{\ve_j}_1} \label{eq:proposal}
\end{align} 
which consists of: (a). local weights (the first term is the normalized
local emissions) where the intuitions is that states with larger local weights are more
likely to be observed and (b). a global prior (the second
term) where the intuition is that larger norms are usually
associated with large dot products, and thus larger potentials.
Note that there is no single fixed recipe of proposal construction since it is usually related to specific parameterization. 
For general structured models,
we recommend considering utilizing any information that is potentially correlated with the actual summands, like the local and global weights we consider here.

\textbf{Exploration-Exploitation Tradeoff}\quad 
It is important to note that gradients only pass through the top $K_1$
indices~$\Omega_t^{K_1}$ and tail samples~$\Omega_t^{K_2}$ during
training.  This means that
 increasing $K_1$  leads to exploiting  states that are already confident enough while 
 increasing $K_2$ leads to exploring less confident states.  So the ratio $K_1 /
 K_2$ also represents an exploration-exploitation tradeoff when
 searching for meaningful posterior structures.
 We will further see its effect in the experiments.

%% file: 041_exp_tab_proposal.tex
\begin{table}[t]
  \caption{
   Proposal effect on Randomized Forward. 
  }
  \label{tab:exp:proposal}
  \small
  \centering
  \begin{tabular}{@{}lccc@{}}
      \toprule
       & Dense & Interm. & Long-tail  \\ \midrule
      \textsc{Uniform} & 0.655 & 11.05 & 0.160\\ 
      \textsc{Local} & 2.908 & 9.171 & 0.481  \\ 
      \textsc{Global} &  0.453 & 0.096 & 0.100 \\ 
      \textsc{Local + Global} & \bf 0.028 & \bf 0.017 & \bf 0.022 \\ 
      \bottomrule
  \end{tabular}
\end{table}

%% file: 041_exp_fig_k1k2.tex
\begin{figure*}[t!]
  \centering
  \includegraphics[width=0.9\linewidth]{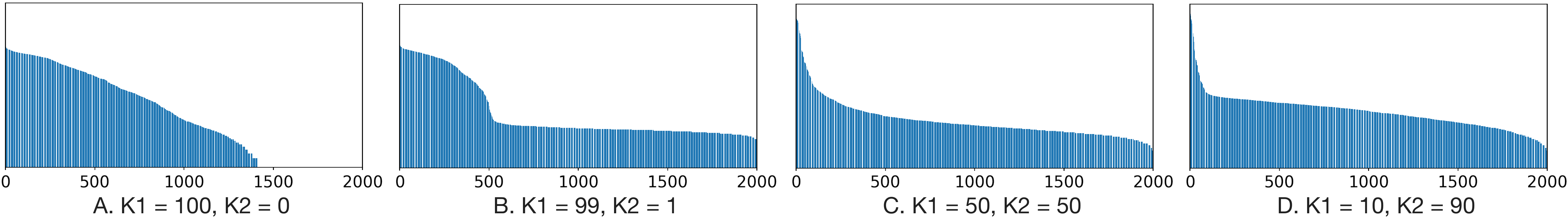}
  \caption{Modulating the posterior by controlling $K_1 / K_2$. 
  The topK only appproach leads to a collapsed posterior, which is overcome by  $K_2$ randomization.   Increasing $K_2$ encourages exploring less certain states, and
  consequently leads to an increasingly heavier tail. 
  }
  \label{fig:control_k1k2}
\end{figure*}

%% file: 040_experiments.tex
We first evaluate the estimation error of individual RDP algorithms
for three unit cases, namely when the underlying distributions are
long-tail, intermediate level, or dense.  Then we evaluate RDP for
training the structured VAE.

\subsection{Evaluation of RDP Algorithms}

\textbf{Experimental Setting}\quad 
We evaluate the mean square error (MSE) for
the estimation of linear-chain log partition (Alg.~\ref{alg:forward}),
hypertree log partition (Alg.~\ref{alg:inside}), and linear-chain
entropy (Alg.~\ref{alg:entropy}) on the three types of distributions.
We simulate the distributions by controlling their entropy: the
smaller the entropy, the more long-tailed the underlying distribution
is.  We set $N$, the number of states, to be 2,000 and 10,000, which is
orders of magnitudes larger than previous
work~\citep{wiseman-etal-2018-learning, yang2021pcfgs}.  We run RDPs
100 times and calculate their MSE against the exact results from the
full DP.  For all estimators, we set $K_2 = 1$ and $K_1 = K - 1$, and
control $K$ to be [1, 10, 20] percent of $N$.
Also note that $K$ is independent of $N$ and we recommend adjusting it according to the computation budget.
For linear-chains, we
use the proposal discussed in Eq.~\ref{eq:proposal}.  For hypertrees,
we use a uniform proposal.

\textbf{Comparison Models}\quad 
Since there is only a handful of
methods that apply to this range of structures and inference, we
choose topK summation \cite{sun2019fast} as our baseline.  This method
was originally developed for linear-chain partition function approximation, and
can be viewed as setting $K_2=0$ in our case (topK summation only, no
sample).  This method is also used in~\citet{correia2020efficient},
however, they only consider sparse underlying distributions.  It is
important to differentiate between \textit{long-tailed} and
\textit{sparse} distributions.  While in both cases head elements
consume a large portion of the probability mass, sparse means that
tail elements have no probability while long-tailed means that tail
elements have non-negligible probability.  This entails that the topK
approach \cite{sun2019fast,correia2020efficient} may significantly
\textit{underestimate} the partition function and entropy.

\input{041_exp_tab_nll.tex}

\textbf{Results}\quad
Table~\ref{tab:exp:overall} shows the MSE results. 
Our method outperforms the topK baseline  on all distributions with significantly less memory budgets.
We highlight two important observations: 
(1) since  TopK summation requires sparsity, it consistently underestimates the partition and entropy, particularly on dense distributions, 
while RDP does not have this limitation; and
(2)~by choosing only \textit{1 percent} of the full index, we are able to outperform TopK summation in many cases and obtain decent estimates.

Table~\ref{tab:exp:biasvar} further shows how the computation budget
of the randomized Forward algorithm can be used  to control the
bias-variance tradeoff.  As can be seen, 
By increasing the computation budget $K$ (which is still smaller
than full size), we are able to simultaneously reduce bias and variance. 
Table~\ref{tab:exp:proposal} shows different proposals influence the
randomized Forward algorithm. We observe that a correlated proposal
(as introduced in Eq.~\ref{eq:proposal}) indeed improves estimates of
the randomized Forward.  In general, scaled models typically require
certain types of decomposition (see examples in~\citet{chiu2021lowrank} and~\citet{yang2021pcfgs}) and proposals are closely related to a
specific parametrization.
For general structures, one can follow the local-global principle to
construct proposals, or retreat to a uniform proposal.

\subsection{Evaluation of the Structured VAE}

We now  use RDP to scale the latent space of
the structured VAE.  This experiment primarily tests the performance
of the randomized Gumbel backward sampling algorithm
(Alg.~\ref{alg:gumbel_ffbs}), which is used as a reparameterized
gradient estimator for training the VAE.  

\textbf{Experimental Setting}\quad
We use an LSTM with 256 dimensional hidden states for the generative
model.  For the inference model, we use a pretrained GPT2 (base size)
to show our method's compatibility with pretrained language models.
We follow \citet{Fu2020GumbelCRF} and use the MSCOCO dataset and reuse
their processed data for simplicity.  Our setting scales the original
Gumbel-CRF paper~\citep{Fu2020GumbelCRF}.  Specifically, we set $N =
[100, 2,000]$.  With $N=100$ we can perform full DP which still gives
biased gradients due to the continuous relaxation.  With $N=2,000$,
full DP gives memory overflow on a 16G GPU, so we only compare to the
TopK approach.  We use $K_1 = K_2 = 10\%N$.  We follow standard
practice and report negative log likelihood estimated by importance
sampling~\citep{kim2019unsupervised}.

\textbf{Results}\quad Table~\ref{tab:exp:nll} shows the performance of
Latent Variable Model training with varying $N$ size. 
We observe that RDP outperforms the baselines for both $N=100$ and
$N=2,000$, which demonstrates its effectiveness for training
VAEs. Furthermore, the overall NLL decreases as~$N$ increases from 100
to 2,000, which underscores the advantage of scaling.
We further show that the ratio of $K_1/K_2$
can be used for modulating the
posterior and prevent posterior collapse.  Specifically, we draw the
frequency of states of the aggregated posterior.  Recall that the
aggregated posterior is obtained by sampling the latent states from
the inference network for all training instances, then calculating the
overall frequency~\citep{mathieu2019disentangling}.
Figure~\ref{fig:control_k1k2} shows that  topK summation
leads to the posterior collapse as there exist multiple inactive
states (frequency~=~0). 
Here it is important to note that gradients only pass through the top $K_1$
indices~$\Omega_t^{K_1}$ and tail samples~$\Omega_t^{K_2}$ during
training.  This means that
increasing $K_1$  leads to exploiting  states that are already confident enough during training, while 
increasing $K_2$ leads to exploring less confident states,
which consequently leads to
increased tail frequency of the aggregated posterior (after
convergence).

%% file: 041_exp_tab_nll.tex
\begin{table}[t]
  \caption{
  Negative Log Likelihood comparison. 
  }
  \label{tab:exp:nll}
  \small
  \centering
  \begin{tabular}{@{}lcc@{}}
      \toprule
       & Dev NLL & Test NLL  \\ \midrule
      \textsc{Full-100} & 39.64 $\pm$ 0.06 & 39.71 $\pm$ 0.07 \\ 
      \textsc{TopK-100} & 39.71 $\pm$ 0.13 & 39.76 $\pm$ 0.11 \\ 
      \textsc{RDP-100} (ours)  & \bf 39.59 $\pm$ 0.10 & \bf 39.59 $\pm$ 0.08\\ \midrule
      \textsc{TopK-2000} (ours)  & 39.81 $\pm$ 0.30 & 39.84 $\pm$ 0.31 \\ 
      \textsc{RDP-2000} (ours)  & \bf 39.47 $\pm$ 0.11 & \bf 39.48 $\pm$ 0.14 \\ 
      \bottomrule
  \end{tabular}
\end{table}

%% file: 050_conclusion.tex
This work proposes 
the randomized sum-product
algorithms for scaling structured models to tens of thousands of
latent states.  Our method is widely applicable to classical DP-based
inference 
and different graph structures.
Our methods reduce computation complexity by orders of
magnitude compared to existing alternatives and is compatible with
automatic differentiation and modern neural networks.  We hope this
method will open new possibilities of large-scale deep structured
prediction models.

%% file: 060_appendix.tex


\section{Biasedness Analysis of Randomized Forward}
\label{sec:app:theory:bias}
\input{061_app_theory_bias.tex}

\section{Variance Analysis of Randomized Forward}
\label{sec:app:theory:variance}
\input{061_app_theory_variance.tex}

%% file: 061_app_theory_bias.tex
In this section, we discuss the unbiasedness and variance of the Randomized Forward algorithm. 
We first show that Randomized Forward gives an unbiased estimator of the partition function. 

\begin{theorem}[Unbiasedness]
  \label{thm:unbiasedness}
  For all $t \in [1, 2, ..., T]$, the sampled sum $\hat{\alpha}_t$ (Eq.~\ref{eq:forward_recursion}) is an unbiased estimator of the forward variable $\alpha_t$. 
  The final sampled sum $\hat{Z}$ is an unbiased estimator of the partition function $Z$. 
\end{theorem}

\begin{proof} By the Second Principle of Mathematical Induction. Assume initialization $\alpha_1(i) = \phi(x_1, i)$. Firstly at $t = 2$, for all $i$, we have:
  \begin{align}
    \mathbb{E}_{q_1}[\hat{\alpha}_2(i)] = 
      \sum_{j = 1}^{K_1} \alpha_1(\sigma_{1, j})\Phi(\sigma_{1, j}, i)\phi(x_2, i) 
      + \frac{1}{K_2} \underbrace{ \sum_{j=1}^{K_2} \mathbb{E}_{q_1}\bigg[\frac{\tilde{Z}_1}{\tilde{q}_1(\delta_{1, j})} \alpha_1(\delta_{1, j})\Phi(\delta_{1, j})\phi(x_2, i)}_{=A} \bigg] \label{eq:app1}
  \end{align}
  where the second term can be expanded as a masked summation with the index rearranged from $\sigma_{2, K_1 + 1}$ to $\sigma_{2, N}$: 
  \begin{align}
    A &= \sum_{j=1}^{K_2} \mathbb{E}_{q_1} \bigg[\frac{\tilde{Z}_1}{\tilde{q}_1(\delta_{1, j})} \alpha_1(\delta_{1, j})\Phi(\delta_{1, j})\phi(x_2, i)  \bigg]\\ 
      &= \sum_{k = 1}^{K_2} \sum_{j = K_1 + 1}^N \mathbb{E}_{q_1}\bigg[ \frac{1}{q_1(\delta_{1, k})}\mathbbm{1}(\delta_{1, k} = j) \alpha_1(\sigma_{1, j})\Phi(\sigma_{1, j})\phi(x_2, i) \bigg] \\ 
      &= \sum_{k = 1}^{K_2} \sum_{j = K_1 + 1}^N \underbrace{\mathbb{E}_{q_1}\bigg[ \frac{1}{q_1(\delta_{1, k})}\mathbbm{1}(\delta_{1, k} = j)\bigg]}_{=1} \alpha_1(\sigma_{1, j})\Phi(\sigma_{1, j})\phi(x_2, i) \\ 
      &= K_2  \sum_{j = K_1 + 1}^N \alpha_1(\sigma_{1, j})\Phi(\sigma_{1, j})\phi(x_2, i)
  \end{align}
  Notice how we re-index the sum. Now put it back to Eq.~\ref{eq:app1} to get:
  \begin{align}
    \mathbb{E}_{q_1}[\hat{\alpha}_2(i)] &= 
      \sum_{j = 1}^{K_1} \alpha_1(\sigma_{1, j})\Phi(\sigma_{1, j}, i)\phi(x_2, i) + \frac{1}{K_2} \cdot K_2  \sum_{j = K_1 + 1}^N \alpha_1(\sigma_{1, j})\Phi(\sigma_{1, j})\phi(x_2, i) \\ 
      &= \sum_{j = 1}^{N} \alpha_1(\sigma_{1, j})\Phi(\sigma_{, j}, i)\phi(x_2, i) \\ 
      &= \sum_{j = 1}^{N} \alpha_1(j)\Phi(j, i)\phi(x_2, i) \\ 
      &= \alpha_2(i)
  \end{align}
  This verifies the induction foundation that $\mathbb{E}_{q_1}[\hat{\alpha}_2(i)] = \alpha_2(i)$ for all $i$. 

  Now assume for all time index up to $t$ we have: 
  \begin{align}
    \forall i, \mathbb{E}_{q_{1:t - 1}}[\hat{\alpha}_t(i)] = \alpha_t(i)
  \end{align}
  Consider $t + 1$, we have: 
  \begin{align}
  \mathbb{E}_{1: q_t}[\hat{\alpha}_{t + 1}(i)] &= 
      \sum_{j = 1}^{K_1} \mathbb{E}_{q_{1:{t-1}}}\big[\hat{\alpha}_t(\sigma_{t, j})\big] \Phi(\sigma_{t, j}, i)\phi(x_{t + 1}, i) \\
      &\quad \quad+ \frac{1}{K_2} \sum_{j=1}^{K_2} \mathbb{E}_{q_{1:t-1}}\big[\hat{\alpha}_t(\delta_{t, j}) \big] \cdot  \mathbb{E}_{q_{t}}\big[ \frac{\tilde{Z}_t}{\tilde{q}_t(\delta_{t, j})} \Phi(\delta_{t, j})\phi(x_{t + 1}, i)] \big] \\ 
      &= \sum_{j = 1}^{K_1} \alpha_t(\sigma_{t, j}) \Phi(\sigma_{t, j}, i)\phi(x_{t + 1}, i) \\
      &\quad \quad+ \frac{1}{K_2} \underbrace{\sum_{j=1}^{K_2} \alpha_t(\sigma_{t, j}) \cdot  \mathbb{E}_{q_{t}}\big[ \frac{\tilde{Z}_t}{\tilde{q}_t(\delta_{t, j})} \Phi(\delta_{t, j})\phi(x_{t + 1}, i)] \big]}_{ = A'}
  \end{align}
  Note how we decompose the expectation by using the independence: $q_{1:t} = q_{1:t-1} \cdot q_t$
  With a similar masked summation trick as $A$, we have:
  \begin{align}
    A' = K_2 \sum_{j = K_1 + 1}^N \alpha_t(\sigma_{t, j}) \Phi(\sigma_{t, j}, i)\phi(x_{t + 1}, i)
  \end{align}
  This gives us:
  \begin{align}
    \mathbb{E}_{1: q_t}[\hat{\alpha}_{t + 1}(i)] &= \sum_{j = 1}^{K_1} \alpha_t(\sigma_{t, j}) \Phi(\sigma_{t, j}, i)\phi(x_{t + 1}, i)  + \frac{1}{K_2} \cdot K_2 \sum_{j = K_1 + 1}^N \alpha_t(\sigma_{t, j}) \Phi(\sigma_{t, j}, i)\phi(x_{t + 1}, i) \\ 
    &= \sum_{j = 1}^{N} \alpha_t(\sigma_{t, j}) \Phi(\sigma_{t, j}, i)\phi(x_{t + 1}, i) \\ 
    &= \sum_{j = 1}^{N} \alpha_t(j) \Phi(j, i)\phi(x_{t + 1}, i) \\ 
    &= \alpha_{t + 1}(i)
  \end{align}
  Thus showing $\hat{\alpha}_{t+1}$ is also an unbiased estimator for $\alpha_{t+1}$ at each step $t+1$. 
  Setting $t = T$, the last step, gives us $\mathbb{E}[\hat{Z}] = Z$ (details similar to the above). 
\end{proof}

\begin{corollary}
  When we change the (sum, product) semiring to the (log-sum-exp, sum) semiring, the expectation of the estimator will become a lower bound of $\log \alpha_t$ and $\log Z$. 
\end{corollary}

\begin{proof}
  Denote $l_t(i) = \log \alpha_t(i)$ and $\Omega$ the set of sampled indices where $|\Omega| = K_2$. For simplicity, we omit the top $K_1$ summation and only show the summation of the sample. 
  Cases where $t > 2$ can be derived similarly as following: 
  \begin{align}
    \hat{l}_{2}(i) &= \log \sum_{j \in \Omega} \exp (l_1(j) + \log \Phi(j, i) + \log \phi(x_t, i)) - \log K_2 \\ 
    \mathbb{E}[\hat{l}_{2}(i)] &= \mathbb{E}\bigg[\log \sum_{j \in \Omega} \exp (l_1(j) + \log \Phi(j, i) + \log \phi(x_t, i))\bigg] - \log K_2 \\ 
    & \le \log \mathbb{E}\bigg[\sum_{j \in \Omega} \exp (l_1(j) + \log \Phi(j, i) + \log \phi(x_t, i))\bigg] \label{eq:app:jensen} - \log K_2\\ 
    &= \log K_2 - \log K_2 + \log \sum_{j \in \Omega} \exp (l_1(j) + \log \Phi(j, i) + \log \phi(x_t, i)) \\ 
    &= l_2(i)
  \end{align}
  where Eq.~\ref{eq:app:jensen} comes from Jensen's inequality. 
  Then by induction one can show at everystep, we have $\mathbb{E}[\hat{l}_t(i)] \le l_t(i)$. 
\end{proof}

Although implementation in the log space makes the estimate biased, it reduces the variance exponentially in a rather trivial way. 
It also provides numerical stability.
So, in practice we use it for training.

%% file: 061_app_theory_variance.tex
Now we analyze variance. We start with the estimator $a_\delta / q_\delta$, $\delta \sim \text{Categorical}\{q_{K_1 +1}, ... q_{K_N}\}$ in Eq.~\ref{eq:tailsum}. Firstly, this is an unbiased estimator of the tail sum:
\begin{align}
  \mathbb{E}[a_\delta / q_\delta] = \sum_{i = K_1 + 1}^N a_i
\end{align}

We have the folling variance arguments:
\begin{theorem}[Variance of the tail estimator]
  \label{thm:tail_variance}
  \begin{align}
    \mathbb{V}[a_\delta / q_\delta] &= \sum_{i=K_1 + 1}^N a_i^2 / q_i - (\sum_{i=K_1 + 1}^N a_i)^2 \label{eq:var_1}\\ 
    &=  S_{K_1}^2 (\sum_{i=K_1 + 1}^N \epsilon_i^2 / q_i + 2\epsilon_i)
  \end{align}
  where 
  \begin{align}
    S_{K_1} =  \sum_{i = K_1 + 1}^N a_i \quad \quad \quad \quad
    \epsilon_i = a_i / S_{K_1} - q_i
  \end{align}
  which says the variance is:
  \begin{itemize}
    \item quadratic to the tail sum $S_{K_1}$, which will can be reduced by increasing $K_1$ as the effect of Rao-Blackwellization.
    \item approximately quadratic to the gap $\epsilon_i$, as the differences between the proposal $q_i$ and the oracle $a_i / S_{K_1}$, which can be reduced by choosing a correlated proposal as the effect of importance sampling.
  \end{itemize}

  Optimal zero variance is achieved on
  \begin{align}
    q_i^* = a_i / S_{K_1} 
  \end{align}
\end{theorem}

\begin{proof}
The variance is then:
\begin{align}
  \mathbb{V}[a_\delta / q_\delta] =  \mathbb{E}[(a_\delta / q_\delta)^2] - \mathbb{E}[a_\delta / q_\delta]^2
\end{align}
where the first term is the second moment: 
\begin{align}
  (a_\delta / q_\delta)^2 &= (\sum_{i = K_1 + 1}^ N a_i \mathbbm{1}[\delta = i] / q_i)^2 \\ 
  &= \sum_{i=K_1 + 1}^N a_i^2 \mathbbm{1}[\delta = i] / q_i^2 + \underbrace{2\sum_{i = K_1 + 1}^N \sum_{j = K_1 + 1}^N \frac{a_i a_j \mathbbm{1}[\delta = i] \mathbbm{1}[\delta = j]}{q_i q_j}}_{=0} \\ 
  &= \sum_{i=K_1 + 1}^N a_i^2 \mathbbm{1}[\delta = i] / q_i^2
\end{align}
taking the expection of it we get:
\begin{align}
  \mathbb{E}[(a_\delta / q_\delta)^2] 
  &= \mathbb{E}[\sum_{i=K_1 + 1}^N a_i^2 \mathbbm{1}[\delta = i] / q_i^2] \\ 
  &= \sum_{i=K_1 + 1}^N a_i^2 / q_i
\end{align}
Plug this back, variance is:
\begin{align}
  \mathbb{V}[a_\delta / q_\delta] = \sum_{i=K_1 + 1}^N a_i^2 / q_i - (\sum_{i=K_1 + 1}^N a_i)^2
\end{align}
To get the optimal proposal, we solve the constrained optimization problem:
\begin{align}
  &\min_{q_i} \sum_{i=K_1 + 1}^N a_i^2 / q_i - (\sum_{i=K_1 + 1}^N a_i)^2 \\ 
  &s.t. \sum_{i = K_1 + 1}^N q_i = 1
\end{align}
By solving the corresponding Lagrangian equation system (omitted here), we get the optimal value achieved at:
\begin{align}
  & q_i^* = \frac{a_i}{\sum_{i = K_1 + 1}^N a_i}
  & \sum_{i=K_1 + 1}^N a_i^2 / q_i^* - (\sum_{i=K_1 + 1}^N a_i)^2 =0 
\end{align}
This says zero variance is achieved by a proposal equal to the normalized summands.

Then define the gap between the proposal and the normalized summands as:
\begin{align}
  \epsilon_i = \frac{a_i}{S_{K_1}} - q_i
\end{align}
Plug this to the variance expression we get: 
\begin{align}
  \mathbb{V}[a_\delta / q_\delta] = S_{K_1}^2 (\sum_{i=K_1 + 1}^N \epsilon_i^2 / q_i + 2\epsilon_i )
\end{align}
\end{proof}

\begin{corollary}
  When increasing the sample size to $K_2$, the variance will reduce to 
  \begin{align}
  \mathbb{V}[\frac{1}{K_2} \sum_{j = 1}^N \frac{a_{\delta_j}}{q_{\delta_j}}] = \frac{1}{K_2} S_{K_1}^2 (\sum_{i=K_1 + 1}^N \epsilon_i^2 / q_i + 2\epsilon_i )
  \end{align}
\end{corollary}


Now we consider the variance of the Sampled Forward algorithm. 
An exact computation would give complicated results.
For simplification, we give an asymptotic argument with regard to Rao-Blackwellization and importance sampling:

\begin{theorem}[Single Step Asymptotic Variance of Sampled Forward]
  \label{thm:sampled_forward_variance}
  At each step, the alpha varible estimator has the following asymptotic variance:
  \begin{align}
    \mathbb{V}[\hat{\alpha}_t(i)] = O\Big(\frac{1}{K_2}\alpha^2_{t, K_1}(i) \cdot \epsilon^2_t(i) \Big) \label{eq:var_alpha_recursion_1}
  \end{align}
  where:
  \begin{itemize}
    \item $\alpha_{t + 1, K_1}(i) = \sum_{j = K_1 + 1}^N \tilde{\alpha}_t(j, i) = \sum_{j = K_1 + 1}^N \sqrt{\mathbb{E}_{q_{1:t-1}}[\hat{\alpha}^2_t(j)]} \Phi(j, i)\phi(x_t, i)$ is a tail sum after the top $K_1$ summands. This term will reduce if we increase $K_1$, as an instance of Rao-Blackwellization.
    \item $\epsilon^2_t(i) = \sum_{j = K_1 + 1}^N \epsilon^2_{t-1}(j, i) / q_{t-1}(j)$ and $\epsilon_{t-1}(j, i)$ is the difference between the proposal $q_{t-1}(j)$ and the oracle proposal. This term will reduce if the proposal is more correlated to the oracle, as an instance of Importance Sampling.
  \end{itemize}
\end{theorem}

\begin{proof}
We start with a simple setting where $K_2 = 1$. At step $t + 1$ we have the following variance recursion:
\begin{align}
  \mathbb{E}_{q_{1:t}}[\hat{\alpha}^2_{t+1}(i)] = \sum_{j=K_1 + 1}^N \frac{\Phi^2(j, i)\phi^2(x_{t + 1}, i)}{q_t(j)} \cdot\mathbb{E}_{q_{1:t-1}}[\hat{\alpha}_t^2(j)]
\end{align}
This is derived by plugging estimator~\ref{eq:forward_recursion} to the variance Eq.~\ref{eq:var_1} we have just derived. Denote:
\begin{align}
  \alpha_{t + 1, K_1}(i) = \sum_{j = K_1 + 1}^N \tilde{\alpha}_t(j, i) = \sum_{j = K_1 + 1}^N \sqrt{\mathbb{E}_{q_{1:t-1}}[\hat{\alpha}^2_t(j)]} \Phi(j, i)\phi(x_{t + 1}, i)
\end{align}
Then we have 
\begin{align}
  \mathbb{V}_{q_{1:t}}[\hat{\alpha}_{t+1}(i)] = \alpha^2_{t + 1, K_1}(i) \big(\sum_{j = K_1 + 1}^N \frac{\epsilon_t^2(j, i)}{q_t(i)} + 2\epsilon_t(j, i) \big)
\end{align}
where $\epsilon_t(j, i)$ is the differences between the proposal and the normalized exact summands at step $t$ state $i$:
\begin{align}
  \epsilon_t(j, i) = \frac{\tilde{\alpha}_t(j, i)}{\sum_{j=K_1 = 1}^N \tilde{\alpha}_t(j, i)} - q_t(j)
\end{align}
Dropping out the first order errors and increasing the number of sample to $K_2$, we have the asymptotics:
\begin{align}
  \mathbb{V}[\hat{\alpha}_t(i)] = O\Big(\frac{1}{K_2}\alpha^2_{t, K_1}(i) \cdot \epsilon^2_t(i) \Big)
\end{align}
\end{proof}


\begin{theorem}[Asymptotic Variance of Sampled Forward Partition Estimation]
  \label{thm:sampled_forward_variance}
  The alpha variable estimators has the following asymptotic variance recurrsion:
  \begin{align}
    \mathbb{V}[\hat{\alpha}_{t + 1}(i)] = O(\frac{1}{K_2} \cdot \phi_{t, K_1}^2 \cdot \epsilon_{t, K_1}^2 \cdot \mathbb{V}[\hat{\alpha}_t]) \label{eq:var_alpha_recursion_2}
  \end{align}
  Compared with Eq.~\ref{eq:var_alpha_recursion_1}, this expression:
  \begin{itemize}
    \item Uses the product of the factors $\phi_{t, K_1}$ (a function of the sum-prod of the factor at step $t$) and the previous step variance $\mathbb{V}[\hat{\alpha}_t]$ to substitute the $\alpha_{t, K_1}$ in equation~\ref{eq:var_alpha_recursion_1}. Again, this term will decrease with a larger $K_1$ (Rao-Blackwellization). 
    \item $\epsilon^2_{t, K_1}(i) = \sum_{j = K_1 + 1}^N \epsilon^2_{t-1}(j, i) / q_{t-1}(j)$ and $\epsilon_{t-1}(j, i)$ is the difference between the proposal $q_{t-1}(j)$ and the oracle proposal. This term will reduce if the proposal is more correlated to the oracle, as an instance of Importance Sampling (same as Eq.~\ref{eq:var_alpha_recursion_1}).
  \end{itemize}
  Consequently, the partition function has the following asymptotic variance:
  \begin{align}
    \mathbb{V}[\hat{Z}] = O(\prod_{t=1}^T \frac{1}{K_2}\cdot \phi_{t, K_1}^2 \cdot \epsilon_{t, K_1}^2)
  \end{align}
  When implemented in the log space, the variance is trivially reduced exponentially: 
  \begin{align}
    \mathbb{V}[\log \hat{Z}] = O(\sum_{t=1}^T \log \frac{1}{K_2}+ 2\log \phi_{t, K_1} + 2\log \epsilon_{t, K_1})
  \end{align}
\end{theorem}

\begin{proof}
  Firstly for simplcity we assume $K_1 = 0$ and $K_2 = 1$. The the estimator variance is: 
  \begin{align}
    \mathbb{V}[\hat{\alpha}_{t + 1}(i)] &= \mathbb{E}_{q_{1:t}}[\hat{\alpha}_{t+1}^2(i)] - \alpha_{t+1}^2(i) \\ 
    &= \mathbb{E}_{q_{1:t-1}} \bigg[\sum_{j = 1}^N \frac{\hat{\alpha}_t^2(j)\Phi^2(j, i)\phi^2(x_t, i)}{q_t(j)} \cdot \frac{\alpha_t^2(j)}{\alpha_t^2(j)} \bigg] - \alpha_{t+1}^2(i)
  \end{align}
  Recall:
  \begin{align}
    \alpha_{t + 1}(i) = \sum_{j=1}^{N} \alpha_t(j)\phi(j, i)\phi(x_t, i)
  \end{align}
  Let: 
  \begin{align}
    \epsilon_t(j, i) = \underbrace{\frac{\alpha_t(j)\phi(j, i)\phi(x_t, i)}{\alpha_{t + 1}(i)}}_{=q(z_t = j| z_{t + 1} = i)} - q_t(j)
  \end{align}
  Then:
  \begin{align}
    \mathbb{V}[\hat{\alpha}_{t + 1}(i)] &= \mathbb{E}_{q_{1:t-1}} \bigg[\sum_{j=1}^N \alpha_{t+1}^2(i)\big( \frac{\epsilon_t^2(j, i)}{q_t(j)} + 2\epsilon_t(j,i) + q_t(j) \big)\frac{\hat{\alpha}_t^2(j)}{\alpha_t^2(j)} \bigg] - \alpha_{t+1}^2(i) \\ 
    &=  \alpha_{t+1}^2(i)\bigg( \sum_{j=1}^N \big( \frac{ \mathbb{E}_{q_{1:t-1}}[\hat{\alpha}_t^2(j)] }{ \alpha_t^2(j) } q_t(j) \big) + \sum_{j=1}^N \big( \frac{\epsilon_t^2(j, i)}{q_t(j)} + 2\epsilon_t(j, i) \big) \frac{\mathbb{E}_{q_{1:t-1}}[\hat{\alpha}_t^2(j)]}{\alpha_t^2(j)} \bigg)\\ 
    & \quad  - \alpha_{t+1}^2(i) 
  \end{align}
  Note that there exist $J_t$ such that:
  \begin{align}
    \frac{ \mathbb{E}_{q_{1:t-1}}[\hat{\alpha}_t^2(j)] }{ \alpha_t^2(j) } <= J_t
  \end{align}
  This is because of the bounded gap of the Jensen's inequality. Also recall:
  \begin{align}
    \sum_{j = 1}^N q_t(j) = 1
  \end{align}
  So we get:
  \begin{align}
    \mathbb{V}[\hat{\alpha}_{t + 1}(i)] &<=  \alpha_{t+1}^2(i)\bigg( J_t - 1 + \sum_{j=1}^N \big( \frac{\epsilon_t^2(j, i)}{q_t(j)} + 2\epsilon_t(j, i) \big) \cdot \big( \frac{\mathbb{V}[\hat{\alpha}_t(j)]}{\alpha_t^2(j)} - 1\big) \bigg) \\ 
    &= \alpha_{t+1}^2(i)\bigg( J_t - 1 - \sum_{j=1}^N \big( \frac{\epsilon_t^2(j, i)}{q_t(j)} + 2\epsilon_t(j, i) \big) \bigg)\\
    &\quad +  \alpha_{t+1}^2(i) \bigg( \sum_{j=1}^N \big( \frac{\epsilon_t^2(j, i)}{q_t(j)} + 2\epsilon_t(j, i) \big) \cdot \big( \frac{\mathbb{V}[\hat{\alpha}_t(j)]}{\alpha_t^2(j)}\big) \bigg)
  \end{align}
  empirically $J_t$ is not the dominate source of variance (but could be in the worst case, depending on the tightness of Jensen's inequality). We focus on the second term: 
  \begin{align}
    \mathbb{V}[\hat{\alpha}_{t + 1}(i)] &= O\bigg(  \alpha_{t+1}^2(i) \bigg( \sum_{j=1}^N \big( \frac{\epsilon_t^2(j, i)}{q_t(j)} + 2\epsilon_t(j, i) \big) \cdot \big( \frac{\mathbb{V}[\hat{\alpha}_t(j)]}{\alpha_t^2(j)}\big) \bigg) \\ 
    &= O\bigg( \Big( \sum_{j = 1}^N \alpha_t(j) \underbrace{\Phi(j, i) \phi(x_t, j)}_{O(\phi_t^2)} \Big)^2 \cdot \Big( \underbrace{\sum_{j=1}^N \big( \frac{\epsilon_t^2(j, i)}{q_t(j)} + 2\epsilon_t(j, i) \big)}_{O(\epsilon_t^2)} \cdot  \frac{1}{\alpha_t^2(j)} \cdot \underbrace{\mathbb{V}[\hat{\alpha}_t(j)]}_{O(\mathbb{V}[\alpha_t])} \Big) \bigg) \\ 
    &= O(\phi_t^2 \cdot \epsilon_t^2 \cdot \mathbb{V}[\alpha_t])
  \end{align} 
  Note that a lot of higher-order sum-products are simplified here. Adding top $K_1$ summation and increasing the sample size to $K_2$ leads to variance reduction as:
  \begin{align}
    \mathbb{V}[\hat{\alpha}_{t + 1}(i)] &= O(\frac{1}{K_2} \cdot \phi_{t, K_1}^2 \cdot \epsilon_{t, K_1}^2 \cdot \mathbb{V}[\alpha_t])
  \end{align}
  Recursively expand this equation we get:
  \begin{align}
    \mathbb{V}[\hat{Z}] = O(\prod_{t=1}^T \frac{1}{K_2}\cdot \phi_{t, K_1}^2 \cdot \epsilon_{t, K_1}^2)
  \end{align}
  Chaning the implementation to the log space we reduce the variance exponentially: 
  \begin{align}
    \mathbb{V}[\log \hat{Z}] = O(\sum_{t=1}^T \log \frac{1}{K_2}+ 2\log \phi_{t, K_1} + 2\log \epsilon_{t, K_1})
  \end{align}
\end{proof}